\documentclass[10pt]{article}

\usepackage{mathrsfs}

\usepackage{amsmath,amsthm,amssymb}
\usepackage{caption}
\usepackage{booktabs}
\usepackage{graphicx}
\usepackage{bbm}
\usepackage{xcolor}
\usepackage[utf8]{inputenc}

\usepackage{tocloft}

\usepackage[T1,OT1]{fontenc}
\usepackage[english]{babel}
\usepackage{enumitem}
\usepackage{hyperref}
\usepackage{color}
\usepackage{colortbl}
\usepackage{tikz}
\usepackage{subfig}
\usepackage{wrapfig}
\usepackage[authoryear,square]{natbib}

\definecolor{gblue}{RGB}{66,133,244}
\definecolor{gred}{RGB}{234, 67, 53}
\definecolor{ggreen}{RGB}{52, 168, 83}
\definecolor{gyellow}{RGB}{251, 188, 5}
\definecolor{nupurple}{RGB}{078,042, 132}
\hypersetup{
    colorlinks=true,
    linkcolor=gblue,
    filecolor=gyellow,      
    urlcolor=gred,
    citecolor=nupurple,
}

\font\tencmmib=cmmib10 \skewchar\tencmmib '60
\newfam\cmmibfam
\textfont\cmmibfam=\tencmmib

\def\lessim{\ \lower4pt\hbox{$
\buildrel{\displaystyle <}\over\sim$}\ }
\def\gessim{\ \lower4pt\hbox{$\buildrel{\displaystyle >}
\over\sim$}\ }

\newcommand{\bet}{\begin{theorem}}
\newcommand{\ent}{\end{theorem}}

\newcommand{\norm}[1]{\left\lVert#1\right\rVert}

\newtheorem{lemma}{\bf Lemma}

\newtheorem{definition}{\bf Definition}
\newtheorem{theorem}{\bf Theorem}
\newtheorem*{theorem*}{\bf Theorem}
\newtheorem*{remark*}{\bf Remark}

\newtheorem{remark}{\bf Remark}
\newtheorem{example}{\bf Example}

\newtheorem{proposition}{\bf Proposition}

\newtheorem{condition}{\bf Condition}

\newenvironment{Proof of lemma}{\noindent{\bf Proof of Lemma}}{\hfill$\Box$\newline}
\newenvironment{Proof of claim}{\noindent{\bf Proof of claim}}{\hfill$\Box$\newline}
\newenvironment{Proof of theorem}{\noindent{\bf Proof of Theorem}}{\hfill\scriptsize{$\Box$}\newline}
\newenvironment{Proof of theorems}{\noindent{\bf Proof of Theorems}}{\hfill$\Box$\newline}
\newenvironment{Proof of proposition}{\noindent{\bf Proof of Proposition}}{\hfill$\Box$\newline}
\newenvironment{Proof of propositions}{\noindent{\bf Proof of Propositions}}{\hfill$\Box$\newline}
\newenvironment{Proof of exercise}{\noindent{\it Proof of Exercise:}}{\hfill$\Box$}
\definecolor{dimmedwhite}{rgb}{0.85,0.85,0.85}

\numberwithin{equation}{section}
\author{Yi Gu\thanks{Department of Mathematics, Northwestern University},
Antonio Auffinger \thanks{Department of Mathematics, Northwestern University}} 
\newcommand{\ki}{\mathbf{Ker}_{\mathbf{I}}}
\newcommand{\ko}{\mathbf{Ker}_{\mathbf{O}}}

\newcommand{\R}{\mathbb R}
\newcommand{\PX}{\mathcal P_X}
\newcommand{\SPX}{\widetilde{\mathcal P}_X}
\newcommand{\dd}{\,d}
\newcommand{\one}{\mathbf 1}
\DeclareMathOperator{\Cov}{Cov}
\DeclareMathOperator{\supp}{supp}

\allowdisplaybreaks[2]

\begin{document}

\title{Equilibrium Distribution for t-Distributed Stochastic Neighbor Embedding with Generalized Kernels}
\maketitle

\begin{abstract}
We extend the results of \citet{auffinger-tsne} on the large-sample limit of t-SNE to a broader class of input and output kernels. We prove that under natural assumptions on the input distribution and the decay of the output kernel, if perplexity scales linearly with data size, the empirical measures and values of minimizers converge to an equilibrium distribution. We conclude with numerical examples comparing different kernels.
\end{abstract}

\tableofcontents
\clearpage

\section{Input and Output Kernels}\label{sec_intro}
t-distributed stochastic neighbor embedding (t-SNE) is a technique for
visualizing high-dimensional data in a low-dimensional space, developed in 2008
by Laurens van der Maaten and Geoffrey Hinton \citep{key}. Its applications
include biological data visualization
\citep{LindermanGeorgeC.2019Fitf,10.1093/bioinformatics/btp035} and many other areas of science (see for instance \cite{murray2024largedatalimitsscaling} for many extra examples).

Roughly speaking, the method uses two kernels to turn distances between data points into
weights: one  for the high-dimensional input (Gaussian) and one  for its low-dimensional
representation (Cauchy). Normalizing these weights gives the probability distributions
that are compared through their relative entropy. The algorithm then seeks to minimize this loss over the choices of outputs.  In this paper, we address the question of the choice of these kernels on the visualization. We follow the perspective developed in \citet{auffinger-tsne}, where instead of analysing the algorithm (that stops after finitely many iterations), one seeks to understand minimizers of the loss.



Let $d>s\geq2$ and $n\geq3$, and let $X_1,\dots,X_n\in\R^d$ and $Y_1,\dots,Y_n\in\R^s$ be input and output points.
\begin{definition}[Generalized input kernel]\label{def:input_kernel}
An input kernel is a positive function
$\ki:\R^d\times\R^d\times(0,\infty)\to(0,\infty)$ satisfying
\begin{align}
\ki(x,x',\sigma)=\ki(x',x,\sigma).
\end{align}
\end{definition}
\begin{definition}[Generalized output kernel]\label{def:output_kernel}
An output kernel is a positive function
$\ko:\R^s\times\R^s\to(0,\infty)$ satisfying
\begin{align}
\ko(y,y')=\ko(y',y).
\end{align}
\end{definition}
The standard choices for t-SNE are
\begin{align}\label{eq:standard-kernels}
\ki(x,x',b)&=\exp\!\left(-\frac{\norm{x-x'}^2}{2b^2}\right),
&\ko(y,y')&=\frac{1}{1+\norm{y-y'}^2}.
\end{align}
We call the version using other kernels \emph{generalized t-SNE}. In the notation below, the Gaussian case in \eqref{eq:standard-kernels} corresponds to $w(t)=t$, $\theta=2$, and $\sigma=(2b^2)^{-1}$.

We impose the following conditions on the kernels.
\begin{condition}\label{input_kernel_conditions}
Fix $\theta>0$. Let $w\in C^2([0,\infty))$ satisfy $w(0)=0$, and set
\begin{align}\label{tsne:input_kernel}
\ki(x,x',\sigma)=\exp\bigl(-w(\sigma\norm{x-x'}^\theta)\bigr).
\end{align}
Assume that
\begin{enumerate}[label=(\arabic*),nosep]
\item $w'(t)>0$ for $t\geq0$, and $w(t)\to\infty$ as $t\to\infty$;
\item $w'(t)+tw''(t)\geq0$ for $t\geq0$;
\item
\begin{align}\label{eq:weighted-input-integrability}
\int_0^\infty r^{d-1}w(r^\theta)e^{-w(r^\theta)}\dd r<\infty.
\end{align}
\end{enumerate}
\end{condition}
Multiplying a kernel by a positive constant does not change the associated probabilities, so the normalization $w(0)=0$ is harmless.

\begin{condition}\label{output_kernel_conditions}
Write $\ko(y,y')=g(y,y')=k(\norm{y-y'})$. Assume that
\begin{enumerate}[label=(\arabic*),nosep]
\item $k\in C^1([0,\infty))$ is positive and nonincreasing, and $k(0)=1$;
\item $k(r)\to0$ as $r\to\infty$;
\item $\sup_{r\geq0}|k'(r)|<\infty$;
\item $k'(0)=0$.
\end{enumerate}
\end{condition}
The normalization $k(0)=1$ also leaves the output probabilities unchanged. For a probability measure $\eta$ on $\R^s$, the normalizer satisfies
\begin{align}\label{eq:positive-normalizer}
0<Z_\eta:=\iint g(y,y')\dd\eta(y)\dd\eta(y')\leq1
\end{align}
since $0<g\leq1$.

\subsection*{Setup and notation}\label{sec:setup}
Fix $\rho\in(0,1)$ and set $\mathrm{Perp}=n\rho$. The conditional probabilities $p_{j\mid i}$ and parameters $\sigma_{n,i}>0$ are determined by the following equations:
\begin{align}\label{eq:conditional-input}
p_{j\mid i}&=\frac{\ki(X_i,X_j,\sigma_{n,i})}
{\sum_{\ell\ne i}\ki(X_i,X_\ell,\sigma_{n,i})},\qquad j\ne i,\\
-\sum_{j\ne i}p_{j\mid i}\log p_{j\mid i}&=\log(\mathrm{Perp})=\log(n\rho).\label{eq:row-entropy}
\end{align}
Here, we restrict attention to $1<n\rho<n-1$, which holds for all sufficiently large $n$. Under the assumptions of Theorem~\ref{tsne:main_theorem_1}, we establish in Section~\ref{tsne:sec:input_kernel_convergence} that, almost surely, \eqref{eq:row-entropy} has a unique solution for each $i$ when $n$ is sufficiently large. 
Set $p_{ii}=0$, and for $i\ne j$ let
\begin{align}\label{eq:symmetric-input}
p_{ij}=\frac{p_{j\mid i}+p_{i\mid j}}{2n}.
\end{align}
Then $p_{ij}=p_{ji}$, $\sum_{j\ne i}p_{j\mid i}=1$, and $\sum_{i\ne j}p_{ij}=1$. For a configuration $Y=(Y_1,\dots,Y_n)$ define
\begin{align}\label{eq:output-probabilities}
q_{ij}(Y)&=\frac{g(Y_i,Y_j)}{\sum_{k\ne\ell}g(Y_k,Y_\ell)},\qquad i\ne j,\\
L_{n,\rho}(X,Y)&=\sum_{i\ne j}p_{ij}\log\frac{p_{ij}}{q_{ij}(Y)},\label{eq:discrete-objective}\\
d_{n,\rho}(X)&=\inf_{Y\in(\R^s)^n}L_{n,\rho}(X,Y).\label{eq:discrete-infimum}
\end{align}

Let $\mu_X$ be a probability measure on $\R^d$ with compact support $K=\supp\mu_X$. Assume that $\dd\mu_X=f\,\dd x$ on $K$, where $f:K\to[0,\infty)$ is continuous. Its zero extension $\widetilde f=f\one_K$ is a density on $\R^d$ and need not be continuous across $\partial K$. Write
\begin{align}\label{eq:empirical-measures}
\nu_n=\frac1n\sum_{i=1}^n\delta_{X_i},\qquad
\mu_n=\frac1n\sum_{i=1}^n\delta_{(X_i,Y_i)}.
\end{align}

For an input probability measure $\nu$, define
\begin{align}\label{eq:input-normalizers}
A_\nu(x,\sigma)&=\int\ki(x,u,\sigma)\dd\nu(u),\\
B_\nu(x,\sigma)&=\int w(\sigma\norm{x-u}^\theta)\ki(x,u,\sigma)\dd\nu(u),\\
F_{\rho,\nu}(x,\sigma)&=-\frac{B_\nu(x,\sigma)}{A_\nu(x,\sigma)}
-\log A_\nu(x,\sigma)+\log\rho.\label{eq:entropy-functional}
\end{align}
Thus $F_{\rho,\nu}(x,\sigma)-\log\rho$ is the relative entropy, with respect to $\nu$, of the probability measure with density $\ki(x,\cdot,\sigma)/A_\nu(x,\sigma)$. We write $\sigma^*_{\rho,\nu}(x)$ for the solution of $F_{\rho,\nu}(x,\sigma)=0$, whenever it exists and is unique. Its existence and continuity are proved in Section~\ref{tsne:sec:existence_uniqueness}.

For a positive function $\psi$ on $K$, let
\begin{align}\label{eq:continuum-input}
p_\psi(x,x')=\frac12\left(
\frac{\ki(x,x',\psi(x))}{\int\ki(x,u,\psi(x))\dd\mu_X(u)}
+\frac{\ki(x',x,\psi(x'))}{\int\ki(x',u,\psi(x'))\dd\mu_X(u)}\right).
\end{align}
When $\sigma^*_{\rho,\mu_X}$ is defined on $K$, write $\sigma(x)=\sigma^*_{\rho,\mu_X}(x)$ and $p=p_\sigma$. For fixed kernels and $\rho$, the function $p$ depends only on the input marginal $\mu_X$.

Let $\PX$ be the set of probability measures on $K\times\R^s$ whose input marginal is $\mu_X$. For $\mu\in\PX$ with output marginal $\mu_Y$, set
\begin{align}\label{eq:continuum-output}
Z_\mu&=\iint g(y,y')\dd\mu_Y(y)\dd\mu_Y(y'),
& q_\mu(y,y')&=g(y,y')/Z_\mu,\\
I_\rho(\mu)&=\iint p(x,x')\log\frac{p(x,x')}{q_\mu(y,y')}
\dd\mu(x,y)\dd\mu(x',y').\label{eq:continuum-objective}
\end{align}
The functional may take the value $+\infty$. The bounds on $p$ proved in Lemma~\ref{lem:input-density} below will justify the decomposition
\begin{align}\label{eq:energy-decomposition}
I_\rho(\mu)&=A_p+\iint p(x,x')V(y,y')\dd\mu\dd\mu+\log Z_\mu,\\
A_p&=\iint p(x,x')\log p(x,x')\dd\mu_X(x)\dd\mu_X(x'),
& V(y,y')&=-\log g(y,y')\geq0.\nonumber
\end{align}
Both $p$ and $q_\mu$ integrate to one with respect to $\mu\times\mu$, so $I_\rho$ is a relative entropy and is nonnegative.

Define the vector-valued logarithmic derivative
\begin{align}\label{eq:force}
h(y,y')=\nabla_y\log g(y,y')=
\begin{cases}
\displaystyle\frac{k'(r)}{k(r)}\frac{y-y'}r,&r=\norm{y-y'}>0,\\[2pt]
0,&y=y'.
\end{cases}
\end{align}
By Condition~\ref{output_kernel_conditions}, $h$ is continuous. Define the set of stationary measures to be
\begin{align}\label{eq:stationary-class}
\SPX=\left\{\mu\in\PX:
\int\bigl(p(x,x')-q_\mu(y,y')\bigr)h(y,y')\dd\mu(x',y')=0
\text{ for }\mu\text{-a.e. }(x,y)\right\},
\end{align}
where the displayed integral is required to be absolutely convergent for $\mu$-almost every $(x,y)$. 

\section{Main Results}\label{tsne:sec:main_results}
The main results require two further assumptions. The first ensures that the equation defining $\sigma(x)$ has a solution for every $x\in K$. 

\begin{condition}\label{hyp:input}
For the given $\rho\in(0,1)$, there is a finite $b>0$ such that
\begin{align}\label{eq:uniform-upper-bracket}
\inf_{x\in K}F_{\rho,\mu_X}(x,b)>0.
\end{align}
\end{condition}
Together with strict monotonicity of $F_{\rho,\mu_X}(x,\cdot)$, Condition~\ref{hyp:input} gives $\sigma(x)<b$ on $K$. Proposition~\ref{prop:global-calibration} verifies this condition for Gaussian, exponential, and quartic input kernels, and gives a sufficient condition for polynomial kernels. 

The second assumption concerns the decay of the output kernel.

\begin{condition}\label{cond:tail}
In addition to Condition~\ref{output_kernel_conditions}, the output kernel satisfies one of the following alternatives:
\begin{enumerate}[label=(\alph*),leftmargin=*,itemsep=4pt]
\item There exists $\beta\geq s$ such that, for every $\lambda>0$,
\begin{align}\label{eq:regular-tail}
\lim_{R\to\infty}\frac{k(\lambda R)}{k(R)}=\lambda^{-\beta}.
\end{align}
\item For every $\lambda>1$,
\begin{align}\label{eq:rapid-tail}
\lim_{R\to\infty}\frac{k(\lambda R)}{k(R)}=0.
\end{align}
\end{enumerate}
\end{condition}
In (a), $k$ is regularly varying with index $-\beta$; in (b), it is rapidly varying. The Cauchy kernel satisfies (a) when $s=2$. Further examples are given in Section~\ref{tsne:sec:generality_conditions}. We are now ready to state our main result, which is an extension of the main theorem from \cite{auffinger-tsne}.

\begin{theorem}\label{tsne:main_theorem_1}
Let $\mu_X$ have compact support $K\subset\R^d$ and a density $f\in C(K)$ in the sense specified above, let $X_i$ be i.i.d. with law $\mu_X$, and fix $d>s\geq2$ and $\rho\in(0,1)$. Assume Conditions~\ref{input_kernel_conditions}, \ref{output_kernel_conditions}, \ref{hyp:input}, and \ref{cond:tail}. With probability one, the following assertions hold. For all sufficiently large $n$, equation~\eqref{eq:row-entropy} has a unique solution for every $i$, and
\begin{align}\label{eq:main-limit}
\lim_{n\to\infty}d_{n,\rho}(X)
=\min_{\mu\in\PX}I_\rho(\mu)
=\min_{\mu\in\SPX}I_\rho(\mu).
\end{align}
Let $\varepsilon_n\downarrow0$ and choose any configurations $Y^{(n)}$ with
\begin{align}\label{eq:near-minima}
L_{n,\rho}(X,Y^{(n)})\leq d_{n,\rho}(X)+\varepsilon_n.
\end{align}
There are translations $a_n\in\R^s$ such that the empirical measures
\begin{align}\label{eq:translated-empirical}
\widehat\mu_n=\frac1n\sum_{i=1}^n\delta_{(X_i,Y_i^{(n)}-a_n)}
\end{align}
are tight. Every weak subsequential limit of this translated sequence is a minimizer in $\PX$, belongs to $\SPX$, and attains both minima in \eqref{eq:main-limit}.
\end{theorem}

\begin{remark}
Contrary to \cite{auffinger-tsne} we do not need to assume that the infimum in \eqref{eq:discrete-infimum} is attained: our results are valid for near minimizers. Note also that the translations $a_n$ are
needed, since $L_{n,\rho}$ is invariant under a common translation of all output points.
\end{remark}
The next result shows that the equilibrium distribution must have compact support.
\begin{theorem}\label{tsne:main_theorem_2}
Under the assumptions of Theorem~\ref{tsne:main_theorem_1}, every global minimizer of $I_\rho$ over $\PX$ has compact support. In particular, all limiting minimizers described in Theorem~\ref{tsne:main_theorem_1} have compact support.
\end{theorem}

We end this section with some recent history of t-SNE and put this result into context.  For many years t-SNE was justified mainly by its empirical success. Rigorous mathematical analysis began only in the late 2010s, first with Shaham and Steinerberger \citep{shaham2017stochasticneighborembeddingseparates} who proved that it separates well-separated clusters. \cite{arora2018analysistsnealgorithmdata} and \cite{linderman2017clusteringtsneprovably} gave the first provable accounts of why the early exaggeration phase recovers clustered structure. \cite{cai2022theoreticalfoundationstsnevisualizing} gave a rigorous description of the distinct phases of t-SNE’s gradient descent.

More recent papers have turned to continuum and asymptotic questions. \cite{auffinger-tsne}, the main inspiration for this paper, studied equilibrium distributions as the number $n$ of points grows when perplexity scales with $n$. Jeong and Wu analyzed convergence of t-SNE as a gradient flow for point clouds on a manifold. \cite{murray2024largedatalimitsscaling} showed that embeddings from the original algorithm have no consistent large-data limit and proposed a rescaled model that does. Calder, Huang, Murray, and Pickarski \cite{calder2026continuumlimittsnedata} proved that, after a natural rescaling, the t-SNE objective is consistent with a continuum variational problem, one whose minimizers are unique and smooth.

\subsection{Outline of the paper} Section~\ref{tsne:sec:existence_uniqueness} proves existence, uniqueness, and continuity of $\sigma^*(x)$. The strategy here is significantly different than \cite{auffinger-tsne}. In Section 4, we prove the main theorem. Section 5 provides examples of kernels and Section 6 adds numerical experiments showcasing different embeddings.

\section{Existence and Uniqueness of \texorpdfstring{$\sigma^*$}{sigma-star}}\label{tsne:sec:existence_uniqueness}
Throughout this section, we assume Condition~\ref{input_kernel_conditions}, and $\mu_X$ has compact support $K$ and density $f\in C(K)$ as above. Set $U=\{x\in\operatorname{int}K:f(x)>0\}$. We write $F=F_{\rho,\mu_X}$ and extend its definition to $\sigma=0$ by $F(x,0)=\log\rho$.

\begin{lemma}\label{lem:entropy-asymptotic}
Under Condition~\ref{input_kernel_conditions}, let
\begin{align}\label{eq:rescaled-constants}
Z_0=\int_{\R^d}e^{-w(\norm{u}^\theta)}\dd u,
\qquad
Z_1=\int_{\R^d}w(\norm{u}^\theta)e^{-w(\norm{u}^\theta)}\dd u.
\end{align}
Then $0<Z_0<\infty$ and $Z_1<\infty$. At every $x\in U$,
\begin{align}\label{eq:entropy-asymptotic}
F(x,\sigma)=\frac d\theta\log\sigma-\log\bigl(f(x)Z_0\bigr)
-\frac{Z_1}{Z_0}+\log\rho+o(1)
\qquad(\sigma\to\infty).
\end{align}
The error is uniform on compact subsets of $U$.
\end{lemma}
\begin{proof}
In polar coordinates, the finiteness of $Z_1$ is exactly \eqref{eq:weighted-input-integrability}. Since $w(t)\to\infty$, we have $w\geq1$ outside a large ball, where the integrand of $Z_0$ is bounded by the integrand of $Z_1$; hence $Z_0<\infty$, and clearly $Z_0>0$.

Now fix $x\in U$ and change variables $u=\sigma^{1/\theta}(x-x')$ in \eqref{eq:input-normalizers}. This gives
\begin{align}
\sigma^{d/\theta}A_{\mu_X}(x,\sigma)
&=\int e^{-w(\norm{u}^\theta)}\widetilde f(x-\sigma^{-1/\theta}u)\dd u,\\
\sigma^{d/\theta}B_{\mu_X}(x,\sigma)
&=\int w(\norm{u}^\theta)e^{-w(\norm{u}^\theta)}
\widetilde f(x-\sigma^{-1/\theta}u)\dd u.
\end{align}
Since $x$ is an interior point of $K$ at which $f$ is continuous, $\widetilde f(x-\sigma^{-1/\theta}u)\to f(x)$ for every $u$ as $\sigma\to\infty$, and $\widetilde f$ is bounded. By dominated convergence the two integrals converge to $f(x)Z_0$ and $f(x)Z_1$ respectively, and both limits are positive because $f(x)>0$. Taking logarithms and substituting into \eqref{eq:entropy-functional} we obtain \eqref{eq:entropy-asymptotic}.
\end{proof}
When $w(t)=t$, Lemma~\ref{lem:entropy-asymptotic} is the analogue of Lemma 1 in \cite{auffinger-tsne}. We only use the asymptotics on $U$; Proposition~\ref{prop:global-calibration} below deals with boundary points and with the zeros of $f$.

\begin{lemma}\label{lem:covariance}
For bounded real functions $A,B$ on a probability space with law $\eta$,
\begin{align}\label{eq:covariance-identity}
\Cov_\eta(A,B)=\frac12\iint(A(u)-A(v))(B(u)-B(v))\dd\eta(u)\dd\eta(v).
\end{align}
In particular, the covariance is nonnegative when the two differences always have the same sign. It is positive if their product is positive on a set of positive $\eta\times\eta$ measure.
\end{lemma}
\begin{proof}
The identity follows by expanding the product;  compare also to \cite{liu2000inequ}.
\end{proof}

\begin{lemma}\label{lem:entropy-monotone}
For every $x\in K$ and $\sigma>0$, $\partial_\sigma F(x,\sigma)>0$.
\end{lemma}
\begin{proof}
Fix $x$, write $R(u)=\norm{x-u}^\theta$, and let
\begin{align}
\dd\eta_{x,\sigma}(u)=\frac{e^{-w(\sigma R(u))}}{A_{\mu_X}(x,\sigma)}\dd\mu_X(u).
\end{align}
Differentiating \eqref{eq:entropy-functional}, the derivative of the explicit logarithm cancels the derivative of $w$ in the numerator, leaving
\begin{align}\label{eq:entropy-derivative}
\partial_\sigma F(x,\sigma)
&=\Cov_{\eta_{x,\sigma}}\bigl(w(\sigma R),R w'(\sigma R)\bigr)\\
&=\frac1\sigma\Cov_{\eta_{x,\sigma}}
\bigl(w(\sigma R),H(\sigma R)\bigr),\qquad H(t)=tw'(t).\nonumber
\end{align}
 By Condition~\ref{input_kernel_conditions}, $w$ is strictly increasing, $H$ is nondecreasing, $H(0)=0$, and $H(t)>0$ for $t>0$.

To see that the inequality is strict, pick $r_2>0$ with $\mu_X\{R\geq r_2\}>0$, which is possible because $\mu_X$ has no atoms, and then $0<r_1<r_2$ so small that $H(\sigma r_1)<H(\sigma r_2)$. Since $x\in K=\supp\mu_X$, the set $\{R\leq r_1\}$ has positive $\mu_X$-measure as well. Both sets have positive $\eta_{x,\sigma}$-measure, and on their product both differences in \eqref{eq:covariance-identity} are strictly positive. Hence $\partial_\sigma F(x,\sigma)>0$.

\end{proof}

\begin{proposition}\label{prop:local-bandwidth}
For every $x\in U$ and $\rho\in(0,1)$, the equation $F(x,\sigma)=0$ has a unique positive solution $\sigma(x)$. The map $x\mapsto\sigma(x)$ is continuous on $U$.
\end{proposition}
\begin{proof}
Fix $x\in U$. We have $F(x,0)=\log\rho<0$, while $F(x,\sigma)\to\infty$ as $\sigma\to\infty$ by Lemma~\ref{lem:entropy-asymptotic}. Since $F(x,\cdot)$ is continuous and, by Lemma~\ref{lem:entropy-monotone}, strictly increasing, there is exactly one root $\sigma(x)$.

For continuity, choose $0<a<\sigma(x)<b$, so that $F(x,a)<0<F(x,b)$. As $F$ is continuous in $x$, these two inequalities persist for all $x'$ in a neighborhood of $x$, and for such $x'$ the root $\sigma(x')$ must lie in $[a,b]$. Now if $x_j\to x$, any limit point of the sequence $\sigma(x_j)$ belongs to $[a,b]$ and, by continuity of $F$, is a root of $F(x,\cdot)$. Hence it equals $\sigma(x)$, and $\sigma$ is continuous at $x$.
\end{proof}
This is the analogue of \cite[Proposition 2.1]{auffinger-tsne}. If $F$ is $C^1$ in $(x,\sigma)$ near $(x,\sigma(x))$, the implicit function theorem gives $C^1$ regularity of $\sigma$ as well, exactly as in that paper.

\begin{proposition}\label{prop:global-calibration}
Assume Condition~\ref{input_kernel_conditions} and let $\mu_X$ and $f\in C(K)$ be as above. Either of the following properties implies Condition~\ref{hyp:input} for every fixed $\rho\in(0,1)$:
\begin{enumerate}[label=(\alph*),leftmargin=*]
\item For every $0<a<b$,
\begin{align}\label{eq:input-separation}
w(\sigma b)-w(\sigma a)\longrightarrow\infty\qquad(\sigma\to\infty).
\end{align}
\item $w(t)=a\log(1+t)$, with $a\theta>d$, and at every $x\in K$,
\begin{align}\label{eq:inverse-mass}
\int_K\norm{x-u}^{-a\theta}\dd\mu_X(u)=\infty.
\end{align}
\end{enumerate}
For (b), a sufficient geometric condition is that, for each $x\in K$, there are $c_x,r_x>0$ and $m_x\leq a\theta$ such that $\mu_X(B(x,r))\geq c_xr^{m_x}$ for $0<r<r_x$.
\end{proposition}
\begin{proof}
Let $\eta_{x,\sigma}$ be the tilted law from the proof of Lemma~\ref{lem:entropy-monotone}. We first show that, under either (a) or (b), $\eta_{x,\sigma}$ concentrates at $x$ as $\sigma\to\infty$, for every $x\in K$.

Assume (a) and fix $0<\epsilon<\delta$. Since $x\in K$, $\mu_X(B(x,\epsilon))>0$. Bounding the kernel from above outside $B(x,\delta)$ and from below inside $B(x,\epsilon)$, we get
\begin{align}
\eta_{x,\sigma}\{\norm{x-u}\geq\delta\}
\leq\frac{\exp[-w(\sigma\delta^\theta)+w(\sigma\epsilon^\theta)]}
{\mu_X(B(x,\epsilon))}\longrightarrow0.
\end{align}
Assume now (b). Multiplying the kernel by $\sigma^a$ changes neither $\eta_{x,\sigma}$ nor $F$, and the new kernel is $(\sigma^{-1}+\norm{x-u}^\theta)^{-a}$. As $\sigma\to\infty$, its integral against $\mu_X$ increases to $\int\norm{x-u}^{-a\theta}\dd\mu_X(u)=\infty$ by monotone convergence and \eqref{eq:inverse-mass}, whereas its integral over the complement of $B(x,\delta)$ is at most $\delta^{-a\theta}$. So in this case too $\eta_{x,\sigma}(B(x,\delta)^c)\to0$.

Next we turn concentration into divergence of the relative entropy. Write
\begin{align}
t_\sigma=\eta_{x,\sigma}(B(x,\delta)),\qquad
b_\delta=\mu_X(B(x,\delta)),\notag
\end{align}
and note that $0<b_\delta<1$ once $\delta$ is small. Splitting the integral defining $D(\eta_{x,\sigma}\Vert\mu_X)$ into the ball and its complement and applying Jensen's inequality on each piece, we find
\begin{align}
D(\eta_{x,\sigma}\Vert\mu_X)
&\geq t_\sigma\log\frac{t_\sigma}{b_\delta}
+(1-t_\sigma)\log\frac{1-t_\sigma}{1-b_\delta}\nonumber\\
&\geq t_\sigma\log(1/b_\delta)-\log2.
\end{align}
Let $\sigma\to\infty$, so that $t_\sigma\to1$, and then $\delta\downarrow0$, so that $b_\delta\to0$ because $\mu_X$ has no atoms. Since $F(x,\sigma)=D(\eta_{x,\sigma}\Vert\mu_X)+\log\rho$, we conclude that $F(x,\sigma)\to\infty$ at every $x\in K$, including boundary points and zeros of $f$.

To get a uniform bound, choose for each $x\in K$ a number $b_x$ with $F(x,b_x)>1$. By continuity of $F$ in $x$, the same inequality holds in a neighborhood of $x$. Extract a finite subcover of $K$ and let $b$ be the largest of the corresponding $b_x$. By Lemma~\ref{lem:entropy-monotone}, $F(x,b)>1$ for all $x\in K$, which is Condition~\ref{hyp:input}.

Finally, for the geometric condition, put $q=a\theta$ and use the layer-cake formula:
\begin{align}
\int_K\norm{x-u}^{-q}\dd\mu_X(u)
\geq q\int_0^{r_x}r^{-q-1}\mu_X(B(x,r))\dd r.\notag
\end{align}
The right side is infinite as soon as $\mu_X(B(x,r))\geq c_xr^{m_x}$ with $m_x\leq q$.
\end{proof}

In particular, $w(t)=t$ satisfies \eqref{eq:input-separation}. Uniform measures on a closed ball or cube satisfy $\mu_X(B(x,r))\geq c_xr^d$ at every support point for sufficiently small $r$. Thus (b) applies to these measures whenever $a\theta>d$, although their zero-extended densities are discontinuous at the boundary.

For an example whose density vanishes at the boundary, take a density proportional to $(1-\norm{x}^2)_+$. For every $x$ in the closed unit ball and all sufficiently small $r$, the intersection of $B(x,r)$ with the unit ball contains a ball of radius proportional to $r$ on which the density is bounded below by a positive constant times $r$. Thus $\mu_X(B(x,r))\geq c_xr^{d+1}$, and (b) applies whenever $a\theta\geq d+1$. In particular, it applies to the exponent $a=d/2+1$, $\theta=2$ used in the simulations. This density is continuous but not $C^1$ across the boundary.

\section{Proofs of the Main Theorems}\label{tsne:sec:proofs}
Throughout this section, $X_i$ are i.i.d. with law $\mu_X$, whose support $K$ is compact and whose density belongs to $C(K)$ as above. We assume Conditions~\ref{input_kernel_conditions}, \ref{output_kernel_conditions}, and \ref{hyp:input}. The uses of Condition~\ref{cond:tail} will be stated explicitly.

\subsection{Convergence of the input probabilities}\label{tsne:sec:input_kernel_convergence}
To account for the omission of $j=i$ in \eqref{eq:row-entropy}, define, for $x\in K$,
\begin{align}\label{eq:leave-one-out}
A_n^-(x,\sigma)&=\int\ki(x,u,\sigma)\dd\nu_n(u)-\frac1n,\\
B_n(x,\sigma)&=\int w(\sigma\norm{x-u}^\theta)\ki(x,u,\sigma)\dd\nu_n(u),\\
F_n^-(x,\sigma)&=-\frac{B_n(x,\sigma)}{A_n^-(x,\sigma)}
-\log A_n^-(x,\sigma)+\log\rho.\label{eq:empirical-entropy}
\end{align}
At $x=X_i$, the subtracted term is $n^{-1}\ki(X_i,X_i,\sigma)=1/n$, and the corresponding term in $B_n$ vanishes. Thus
\begin{align}\label{eq:finite-entropy-identity}
F_n^-(X_i,\sigma)=\log(n\rho)
+\sum_{j\ne i}p_{j\mid i}(\sigma)\log p_{j\mid i}(\sigma).
\end{align}
The function $F_n^-$ is defined where $A_n^->0$. On each set $K\times[0,b]$, this holds for all sufficiently large $n$, as shown below.

\begin{proposition}\label{prop:uniform-input}
There are $0<a<b<\infty$, independent of $n$, such that almost surely, for all sufficiently large $n$, the row-entropy equations have unique solutions $\sigma_{n,i}\in[a,b]$ and the continuum roots are defined on all of $K$ with $\sigma(x)\in[a,b]$. Moreover,
\begin{align}\label{eq:uniform-entropy-convergence}
\sup_{x\in K,\,\sigma\in[a,b]}|F_n^-(x,\sigma)-F(x,\sigma)|&\longrightarrow0,\\
\sup_{1\leq i\leq n}|\sigma_{n,i}-\sigma(X_i)|&\longrightarrow0.\label{eq:uniform-bandwidth-convergence}
\end{align}
There is a constant $m>0$ such that, eventually,
\begin{align}\label{eq:inverse-estimate}
\sup_i|\sigma_{n,i}-\sigma(X_i)|
\leq\frac1m\sup_{x\in K,\,\sigma\in[a,b]}|F_n^-(x,\sigma)-F(x,\sigma)|.
\end{align}
\end{proposition}
\begin{proof}
Choose $b$ as in Condition~\ref{hyp:input}. Since $F(x,0)=\log\rho<0$ and $K$ is compact, continuity gives $a\in(0,b)$ such that $\sup_{x\in K}F(x,a)<0$. By Lemma~\ref{lem:entropy-monotone}, there is a unique root in $[a,b]$ for every $x\in K$. The argument in Proposition~\ref{prop:local-bandwidth} proves that $\sigma$ is continuous on $K$.

Let $D=\operatorname{diam}K$. On $K\times K\times[0,b]$,
$\ki(x,u,\sigma)\geq e^{-w(bD^\theta)}>0$. The integrands defining $A$ and $B$ are bounded and uniformly continuous on this compact set. Given $\epsilon>0$, choose finitely many parameter pairs $(x,\sigma)$ so that every integrand differs from one of the corresponding integrands by at most $\epsilon$, uniformly in $u$. The strong law at those pairs then gives uniform convergence of the empirical integrals, with error at most $2\epsilon+o(1)$. Applying this for $\epsilon=1,1/2,\ldots$ gives almost sure uniform convergence. The subtraction of $1/n$ does not affect the limit, and the denominators are bounded away from zero. The same convergence therefore holds for $F_n^-$.

The same argument applies to the derivatives in $\sigma$. With $R=\norm{x-u}^\theta$, their integrands are linear combinations of $Rw'(\sigma R)e^{-w(\sigma R)}$ and $w(\sigma R)Rw'(\sigma R)e^{-w(\sigma R)}$, which are again bounded and continuous. Hence $\partial_\sigma F_n^-$ converges uniformly to $\partial_\sigma F$ on $K\times[a,b]$. Lemma~\ref{lem:entropy-monotone} and compactness give
\begin{align}
m:=\min_{x\in K,\,\sigma\in[a,b]}\partial_\sigma F(x,\sigma)>0.
\end{align}
For all sufficiently large $n$, $\partial_\sigma F_n^-\geq m/2$ on this set and $F_n^-(X_i,a)<0<F_n^-(X_i,b)$ for every $i$. Thus each equation has exactly one root in $[a,b]$.

For every $\sigma>0$, the covariance calculation in \eqref{eq:entropy-derivative}, applied to the probabilities on the $n-1$ points other than $X_i$, gives $\partial_\sigma F_n^-(X_i,\sigma)\geq0$. The signs at $a$ and $b$ therefore exclude roots outside $[a,b]$. Finally, since $F_n^-(X_i,\sigma_{n,i})=F(X_i,\sigma(X_i))=0$, the mean value theorem gives
\begin{align}
m|\sigma_{n,i}-\sigma(X_i)|
\leq |F(X_i,\sigma_{n,i})-F_n^-(X_i,\sigma_{n,i})|,\notag
\end{align}
which proves \eqref{eq:inverse-estimate}.
\end{proof}

Define, on $K\times K$,
\begin{align}\label{eq:input-density-notation}
a(x,x')&=\ki(x,x',\sigma(x)),
&p'(x,x')&=\frac{a(x,x')}{\int a(x,u)\dd\mu_X(u)},\\
p(x,x')&=\frac12\bigl(p'(x,x')+p'(x',x)\bigr).\nonumber
\end{align}
For all sample indices $i,j$, including $i=j$, define
\begin{align}\label{eq:scaled-input-array}
p'_n(X_i,X_j)&=\frac{\ki(X_i,X_j,\sigma_{n,i})}
{n^{-1}\sum_{\ell\ne i}\ki(X_i,X_\ell,\sigma_{n,i})},\\
p_n(X_i,X_j)&=\frac12\bigl(p'_n(X_i,X_j)+p'_n(X_j,X_i)\bigr).\nonumber
\end{align}
For $i\ne j$, $p_{j\mid i}=p'_n(X_i,X_j)/n$ and $p_{ij}=p_n(X_i,X_j)/n^2$. Although $p_{ii}=0$, the definition of $p_n(X_i,X_i)$ above will be convenient when writing sums as integrals against empirical measures.

\begin{lemma}\label{lem:input-density}
Almost surely,
\begin{align}\label{eq:input-density-convergence}
\max_{i,j}|p'_n(X_i,X_j)-p'(X_i,X_j)|&\longrightarrow0,\\
\delta_n:=\max_{i,j}|p_n(X_i,X_j)-p(X_i,X_j)|&\longrightarrow0.\nonumber
\end{align}
There are constants $0<c<C<\infty$ such that $c\leq p(x,x')\leq C$ on $K^2$ and, eventually, $c\leq p_n(X_i,X_j)\leq C$ for all indices. The function $p$ is continuous on $K^2$.
\end{lemma}
\begin{proof}
Since $\sigma_{n,i},\sigma(x)\in[a,b]$, the numerators are bounded above by one and below by a positive constant. The denominators have the same lower bound, up to a factor $(n-1)/n$. Proposition~\ref{prop:uniform-input}, uniform continuity of $\ki$, and the uniform strong law used in its proof give convergence of the numerators and denominators, uniformly in $i,j$. The assertions follow by taking ratios and symmetrizing. Continuity of $p$ follows from continuity of $\sigma$.
\end{proof}

\begin{lemma}\label{lem:input-entropy}
Almost surely,
\begin{align}\label{eq:input-entropy-limit}
\sum_{i\ne j}p_{ij}\log p_{ij}=A_p-2\log n+o(1).
\end{align}
\end{lemma}
\begin{proof}
Since $\sum_{i\ne j}p_{ij}=1$, the left side plus $2\log n$ equals
\begin{align*}
\frac1{n^2}\sum_{i\ne j}p_n(X_i,X_j)\log p_n(X_i,X_j).
\end{align*}
The uniform bounds and convergence in Lemma~\ref{lem:input-density} allow replacement by $p\log p$ with error $o(1)$. Adding its bounded diagonal costs $O(1/n)$. Finally, $\nu_n\times\nu_n$ converges weakly to $\mu_X\times\mu_X$ and $p\log p$ is bounded and continuous.
\end{proof}

\subsection{Tightness of the empirical measures}\label{sec:output-compactness}
This is the part of the argument where the output kernel enters, and where we need Condition~\ref{cond:tail}. Put $v(r)=-\log k(r)$, so that $V(y,y')=v(\norm{y-y'})$. For $n$ output points let $N_n=n(n-1)$ and define the following averages over ordered pairs of distinct indices:
\begin{align}\label{eq:pair-gap}
T_n&=\frac1{N_n}\sum_{i\ne j}v(\norm{Y_i-Y_j}),
&\overline Z_n&=\frac1{N_n}\sum_{i\ne j}k(\norm{Y_i-Y_j}),\\
J_n&=T_n+\log\overline Z_n\geq0.\nonumber
\end{align}
The last inequality is Jensen's inequality. Equivalently, $J_n=D(r_n\Vert q)$, where $r_n(i,j)=1/N_n$ is the uniform probability on ordered distinct pairs.

\begin{lemma}\label{lem:entropy-mixture}
If $p_{ij}\geq c/N_n$ for all $i\ne j$ and some $0<c<1$, then
\begin{align}\label{eq:mixture-bound}
L_{n,\rho}(X,Y)\geq cJ_n-H_2(c),
\qquad H_2(c)=-c\log c-(1-c)\log(1-c).
\end{align}
\end{lemma}
\begin{proof}
Write $p=cr_n+(1-c)s_n$, where $s_n=(p-cr_n)/(1-c)$ is a probability vector by assumption. In the sum $\sum p\log(p/q)$ we use $p\geq cr_n$ in the terms coming from $cr_n$ and $p\geq(1-c)s_n$ in the terms coming from $(1-c)s_n$ (with the convention $0\log0=0$ for the zero entries of $s_n$). This gives
\begin{align}
D(p\Vert q)&\geq cD(r_n\Vert q)+(1-c)D(s_n\Vert q)
+c\log c+(1-c)\log(1-c)\\
&\geq cJ_n-H_2(c),
\end{align}
since relative entropies are nonnegative.
\end{proof}

\begin{lemma}\label{lem:riesz}
For any probability measure $\eta$ on $\R^s$ and any $\beta\geq s$,
\begin{align}\label{eq:riesz-infinite}
\iint\norm{u-v}^{-\beta}\dd\eta(u)\dd\eta(v)=\infty.
\end{align}
\end{lemma}
\begin{proof}
Pick a bounded cube $Q$ with $m=\eta(Q)>0$. For every sufficiently small $r>0$ we can partition $Q$ into at most $Cr^{-s}$ measurable cells of diameter less than $r$. If $m_i$ are their masses, then by Cauchy--Schwarz
\begin{align}
(\eta\times\eta)\{\norm{u-v}<r\}\geq\sum_i m_i^2\geq \frac{m^2}{Cr^{-s}}=cr^s.
\end{align}
By the layer-cake formula, the energy in \eqref{eq:riesz-infinite} is at least a positive constant times $\int_0^{r_0}r^{s-\beta-1}\dd r$, which diverges because $\beta\geq s$. Note that $\eta$ is allowed to have atoms.
\end{proof}

\begin{lemma}\label{lem:pair-tightness}
Let $\eta_j$ be probabilities on $\R^s$. If the laws of $u-v$ under $\eta_j\times\eta_j$ are tight, then there are $a_j\in\R^s$ such that the translated probabilities $(u\mapsto u-a_j)_\#\eta_j$ are tight.
\end{lemma}
\begin{proof}
Fix $R_0$ such that $(\eta_j\times\eta_j)\{\norm{u-v}\leq R_0\}\geq3/4$ for all $j$. Averaging over $v$, for each $j$ there is a point $a_j$ with $\eta_j(\overline B(a_j,R_0))\geq2/3$. Now if $\norm{u-a_j}>R+R_0$ and $v\in\overline B(a_j,R_0)$, then $\norm{u-v}>R$. Hence
\begin{align}
\eta_j\{\norm{u-a_j}>R+R_0\}
\leq\tfrac32(\eta_j\times\eta_j)\{\norm{u-v}>R\},
\end{align}
and the right side tends to zero as $R\to\infty$, uniformly in $j$.
\end{proof}

\begin{lemma}\label{lem:gap-coercivity}
Assume Conditions~\ref{output_kernel_conditions} and \ref{cond:tail}. For any sequence $n_j\to\infty$ and output configurations in $\R^s$, boundedness of $J_{n_j}$ implies boundedness of $T_{n_j}$.
\end{lemma}
\begin{proof}
Suppose the claim fails. After passing to a subsequence, which we still index by $n$, we may assume that $T_n\to\infty$ while $J_n\leq M$ for some fixed $M\geq0$. Let $(i,j)$ be a uniformly random ordered pair of distinct indices and put $U_n=v(\norm{Y_i-Y_j})$, so that
\begin{align}\label{eq:gap-deviation}
\mathbb E U_n=T_n,
\qquad \mathbb E e^{T_n-U_n}=e^{J_n}\leq e^M.
\end{align}
By Markov's inequality, $\mathbb P\{T_n-U_n>t\}\leq e^{M-t}$, and integrating the tail,
\begin{align}\label{eq:potential-deviations}
\mathbb E(T_n-U_n)_+
&\leq\int_0^\infty\min\{1,e^{M-t}\}\dd t=M+1.
\end{align}
Since $\mathbb E(U_n-T_n)=0$, we also have $\mathbb E(U_n-T_n)_+=\mathbb E(T_n-U_n)_+\leq M+1$.

Discarding finitely many $n$, we may assume $T_n>0$. As $v$ is continuous with $v(0)=0$ and $v(r)\to\infty$, we can choose $R_n>0$ with $v(R_n)=T_n$, and $R_n\to\infty$ because $T_n\to\infty$. For $\lambda>1$ and $n$ large enough that $v(\lambda R_n)>v(R_n)$, monotonicity of $v$ and Markov's inequality give
\begin{align}\label{eq:rescaled-distance-tail}
\mathbb P\{\norm{Y_i-Y_j}>\lambda R_n\}
\leq\mathbb P\{U_n-T_n\geq v(\lambda R_n)-v(R_n)\}
\leq\frac{M+1}{v(\lambda R_n)-v(R_n)}.
\end{align}
Under \eqref{eq:regular-tail} the denominator converges to $\beta\log\lambda$, and under \eqref{eq:rapid-tail} it tends to infinity. In both cases
\begin{align}
\lim_{\lambda\to\infty}\limsup_{n\to\infty}
\mathbb P\{\norm{Y_i-Y_j}>\lambda R_n\}=0,\notag
\end{align}
and after enlarging $\lambda$ to take care of finitely many $n$, we conclude that the laws of $(Y_i-Y_j)/R_n$ are tight.

Consider now the rescaled empirical measures $\eta_n=n^{-1}\sum_i\delta_{Y_i/R_n}$. Under $\eta_n\times\eta_n$, the law of $u-v$ is the law we just considered mixed with an atom of mass $1/n$ at the origin, hence tight. By Lemma~\ref{lem:pair-tightness} and Prokhorov's theorem, after translating and passing to a further subsequence, $\eta_n\Rightarrow\eta$ for some probability measure $\eta$ on $\R^s$. We keep writing $\eta_n$ for the translated measures and $u_i$ for their atoms, so that $\|Y_i-Y_j\|=R_n\|u_i-u_j\|$, while $J_n$ still refers to the original configuration. Let
\begin{align}
\Lambda_n=\frac1{N_n}\sum_{i\ne j}\delta_{(u_i,u_j)}.\notag
\end{align}
For every bounded continuous $\varphi$,
\begin{align}
\left|\int\varphi\dd\Lambda_n-\int\varphi\dd(\eta_n\times\eta_n)\right|
\leq\frac{2\norm\varphi_\infty}{n},\notag
\end{align}
so $\Lambda_n\Rightarrow\eta\times\eta$ as well.

Set $K_n(r)=k(R_nr)/k(R_n)$. Then, by \eqref{eq:pair-gap} and \eqref{eq:gap-deviation},
\begin{align}\label{eq:rescaled-energy-bound}
\int K_n(\norm{u-v})\dd\Lambda_n(u,v)=\frac{\overline Z_n}{k(R_n)}=\overline Z_ne^{T_n}=e^{J_n}\leq e^M.
\end{align}
We now distinguish the two alternatives of Condition~\ref{cond:tail}.

\emph{Case (a).} Here $K_n(r)\to r^{-\beta}$ for every $r>0$. Since each $K_n$ is nonincreasing and the limit is continuous, the convergence is uniform on compact subintervals of $(0,\infty)$: partition the interval finely enough that $r^{-\beta}$ oscillates little on each piece, and compare $K_n$ with its values at the endpoints. For $L>0$ let
\begin{align}
\phi_L(r)=\min\{L,r^{-\beta}\}\quad(r>0),\qquad \phi_L(0)=L.
\end{align}
Given $\epsilon>0$, choose $\delta>0$ and $B>\delta$ with $\delta^{-\beta}>L$ and $B^{-\beta}<\epsilon$. For large $n$ we have $K_n(\delta)>L$, and uniform convergence gives $K_n\geq\phi_L-\epsilon$ on $[\delta,B]$. On $[0,\delta]$ the same inequality follows from the monotonicity of $K_n$, and on $[B,\infty)$ from $K_n\geq0>\phi_L-\epsilon$. Thus
\begin{align}
K_n(r)\geq\phi_L(r)-\epsilon\qquad(r\geq0)\notag
\end{align}
for all large $n$. Since $(u,v)\mapsto\phi_L(\norm{u-v})$ is bounded and continuous, \eqref{eq:rescaled-energy-bound} and $\Lambda_n\Rightarrow\eta\times\eta$ give
\begin{align}
\iint\phi_L(\norm{u-v})\dd\eta(u)\dd\eta(v)\leq e^M+\epsilon.\notag
\end{align}
Let $\epsilon\downarrow0$ and then $L\uparrow\infty$. By monotone convergence,
\begin{align}
\iint\norm{u-v}^{-\beta}\dd\eta(u)\dd\eta(v)\leq e^M,
\end{align}
which contradicts Lemma~\ref{lem:riesz} because $\beta\geq s$.

\emph{Case (b).} Applying \eqref{eq:rapid-tail} with $R=R_n/2$ and $\lambda=2$, we get $K_n(1/2)=k(R_n/2)/k(R_n)\to\infty$. Since $K_n$ is nonincreasing, \eqref{eq:rescaled-energy-bound} forces
$\Lambda_n\{\norm{u-v}\leq1/2\}\leq e^M/K_n(1/2)\to0$.
By the Portmanteau theorem applied to the open set $\{\norm{u-v}<1/2\}$,
\begin{align}
(\eta\times\eta)\{\norm{u-v}<1/2\}
\leq\liminf_n\Lambda_n\{\norm{u-v}<1/2\}=0.\notag
\end{align}
But $\eta$ gives positive mass to some ball of radius $1/8$, and any two points of that ball are less than $1/2$ apart, so the left side is positive. This contradiction completes the proof.
\end{proof}

\begin{proposition}\label{prop:output-tightness}
Suppose that the conclusions of Lemma~\ref{lem:input-density} hold and that Condition~\ref{cond:tail} is satisfied. Every sequence of output configurations for which $L_{n,\rho}(X,Y^{(n)})$ is bounded for sufficiently large $n$ satisfies
\begin{align}\label{eq:uniform-potential-bound}
\sup_{n\text{ sufficiently large}}\frac1{n(n-1)}\sum_{i\ne j}
v(\norm{Y_i-Y_j})<\infty.
\end{align}
Translations can be chosen so that their output empirical measures, and hence their joint empirical measures, are tight. In particular this applies to every sequence satisfying \eqref{eq:near-minima}.
\end{proposition}
\begin{proof}
By Lemma~\ref{lem:input-density}, $p_{ij}=p_n(X_i,X_j)/n^2\geq c_0/n^2$ for $n$ large, so after decreasing $c_0$ to some $c\in(0,1)$ we have $p_{ij}\geq c/[n(n-1)]$ for all $i\neq j$. Lemma~\ref{lem:entropy-mixture} then bounds $J_n$ in terms of $L_{n,\rho}(X,Y^{(n)})$, and Lemma~\ref{lem:gap-coercivity} bounds $T_n$. This is \eqref{eq:uniform-potential-bound}.

For the second claim, average \eqref{eq:uniform-potential-bound} over the first index: for every $n$ there is an index $i_n$ with
\begin{align}\label{eq:anchor-bound}
\frac1n\sum_jv(\norm{Y_j-Y_{i_n}})\leq\frac{n-1}{n}T_n\leq B.
\end{align}
Translate by $a_n=Y_{i_n}$. For $R$ with $v(R)>0$, Markov's inequality shows that the translated empirical output measure puts mass at most $B/v(R)$ outside the ball of radius $R$, and $v(R)\to\infty$. This is tightness of the output empirical measures (the finitely many $n$ not covered by \eqref{eq:uniform-potential-bound} can be added without harm), and since the inputs live in the compact set $K$, the joint empirical measures are tight as well. The last assertion follows from \eqref{eq:collapsed-comparison} below, which shows that $d_{n,\rho}(X)$ is bounded.
\end{proof}

\subsection{Lower bound for the optimal values}\label{tsne:sec:lower_bound}
The objective is nonnegative and finite for every finite configuration. Taking all $Y_i$ equal gives
\begin{align}\label{eq:collapsed-comparison}
0\leq d_{n,\rho}(X)
\leq\sum_{i\ne j}p_{ij}\log p_{ij}+\log\bigl(n(n-1)\bigr)
=A_p+o(1).
\end{align}
Thus approximate minimizers satisfy the boundedness assumption in Proposition~\ref{prop:output-tightness}.

For a joint empirical measure $\mu_n$, use the array $p_n$ in \eqref{eq:scaled-input-array} and define
\begin{align}\label{eq:scaled-output-array}
Z_n&=\frac1{n^2}\sum_{i\ne j}g(Y_i,Y_j)
=\iint g\dd\mu_n\dd\mu_n-\frac1n,
& q_n(y,y')&=g(y,y')/Z_n.
\end{align}
Thus $q_{ij}=q_n(Y_i,Y_j)/n^2$ for $i\ne j$. The normalizer in \eqref{eq:pair-gap} is related to this one by $Z_n=(1-1/n)\overline Z_n$.

\begin{lemma}\label{lem:liminf}
Suppose that $\nu_n\Rightarrow\mu_X$ and that the conclusions of Lemmas~\ref{lem:input-density} and \ref{lem:input-entropy} hold. Let $Y^{(n)}$ be any configurations such that $\mu_n\Rightarrow\mu$. Then $\mu\in\PX$ and
\begin{align}\label{eq:liminf}
\liminf_{n\to\infty}L_{n,\rho}(X,Y^{(n)})\geq I_\rho(\mu).
\end{align}
\end{lemma}
\begin{proof}
The input marginals converge to $\mu_X$, giving $\mu\in\PX$. Bounded continuity of $g$ and convergence of the product measures imply
\begin{align}\label{eq:normalizer-limit}
Z_n=\iint g\dd\mu_n\dd\mu_n-1/n\longrightarrow Z_\mu>0.
\end{align}
The exact discrete decomposition is
\begin{align}\label{eq:discrete-decomposition}
L_{n,\rho}(X,Y^{(n)})
= A_n+\frac1{n^2}\sum_{i\ne j}p_n(X_i,X_j)V(Y_i,Y_j)+\log Z_n,
\end{align}
where $A_n=n^{-2}\sum_{i\ne j}p_n(X_i,X_j)\log p_n(X_i,X_j)\to A_p$.

For large $n$, Lemma~\ref{lem:input-density} gives
$p_n(X_i,X_j)\geq(1-\delta_n/c)p(X_i,X_j)$. Since $V\geq0$ and $V(y,y)=0$, the attractive term is at least
$(1-\delta_n/c)\iint p(x,x')V(y,y')\dd\mu_n\dd\mu_n$.
For each $M$, the integrand $p\min\{V,M\}$ is bounded and continuous. Weak convergence followed by monotone convergence in $M$ gives the required lower bound for this term, including when the limit is infinite. Combining it with $A_n\to A_p$ and \eqref{eq:normalizer-limit} proves \eqref{eq:liminf}.
\end{proof}
\subsection{Upper bound for the optimal values}\label{tsne:sec:upper_bound}
We first prove the upper bound for a measure with compact support.
\begin{lemma}\label{lem:compact-recovery}
For every fixed $\mu\in\PX$ with compact support, almost surely in the input sequence,
\begin{align}\label{eq:compact-recovery}
\limsup_{n\to\infty}d_{n,\rho}(X)\leq I_\rho(\mu).
\end{align}
\end{lemma}
\begin{proof}
Sample i.i.d. pairs $(X_i,Y_i)$ with law $\mu$. Their empirical measures converge weakly to $\mu$ almost surely. All output points belong to a fixed compact set, where $V$ is bounded and continuous. Lemma~\ref{lem:input-density}, the vanishing diagonal terms, and weak convergence therefore give convergence of the attractive term in \eqref{eq:discrete-decomposition}; the other terms converge by Lemma~\ref{lem:input-entropy} and \eqref{eq:normalizer-limit}. Thus $L_{n,\rho}(X,Y)\to I_\rho(\mu)$ for almost every joint sample.

The inputs $X_i$ are i.i.d. with law $\mu_X$. Conditioning on them shows that, for almost every input sequence, there exists an output sequence for which the convergence above holds. The infimum over configurations is no larger than the value at this competitor, proving \eqref{eq:compact-recovery}.
\end{proof}

\begin{lemma}\label{lem:compact-approximation}
For every $\mu\in\PX$ with $I_\rho(\mu)<\infty$, there are compactly supported measures $\mu_R\in\PX$ such that $I_\rho(\mu_R)\to I_\rho(\mu)$.
\end{lemma}
\begin{proof}
Let $\Pi_R$ be Euclidean projection onto the closed output ball of radius $R$, and let $\mu_R$ be the law of $(x,\Pi_Ry)$ when $(x,y)$ has law $\mu$. Projection is nonexpansive: its variational inequalities give
$\norm{\Pi_Ry-\Pi_Ry'}^2\leq(\Pi_Ry-\Pi_Ry')\cdot(y-y')$, and Cauchy--Schwarz gives the distance bound. Since $k$ is nonincreasing,
\begin{align}
0\leq V(\Pi_Ry,\Pi_Ry')\leq V(y,y'),\qquad
V(\Pi_Ry,\Pi_Ry')\longrightarrow V(y,y').
\end{align}
The input marginal, and hence $p$ and $A_p$, do not change. Finite energy implies integrability of $pV$ by \eqref{eq:energy-decomposition}. Dominated convergence gives convergence of the attractive integrals. Bounded convergence gives $Z_{\mu_R}\to Z_\mu>0$, hence convergence of their logarithms. The energy decomposition proves the result.
\end{proof}

\begin{proposition}\label{prop:limsup}
Almost surely,
\begin{align}\label{eq:limsup}
\limsup_{n\to\infty}d_{n,\rho}(X)\leq m,
\qquad m:=\inf_{\mu\in\PX}I_\rho(\mu)<\infty.
\end{align}
\end{proposition}
\begin{proof}
Since $I_\rho(\mu_X\times\delta_0)=A_p$, we have $0\leq m\leq A_p<\infty$. By Lemma~\ref{lem:compact-approximation}, for each $j\geq1$ we may choose a deterministic measure $\lambda_j\in\PX$ with compact support and $I_\rho(\lambda_j)\leq m+1/j$. Lemma~\ref{lem:compact-recovery} applies to each $\lambda_j$. On the intersection of the resulting countably many events of probability one,
\begin{align}
\limsup_{n\to\infty}d_{n,\rho}(X)\leq I_\rho(\lambda_j)\leq m+1/j
\qquad(j\geq1).\notag
\end{align}
Let $j\to\infty$.
\end{proof}

\subsection{Compact support and stationarity of minimizers}\label{sec:support-proof}
We now prove that minimizers of $I_\rho$ have compact support. The proof uses variations that preserve the input marginal.

\begin{proposition}\label{prop:compact-support}
Assume Condition~\ref{output_kernel_conditions} and $0<c\leq p\leq C$ on $K^2$. Every finite-energy global minimizer $\mu$ of $I_\rho$ over $\PX$ has compact support.
\end{proposition}
\begin{proof}
Since $I_\rho(\mu)<\infty$, \eqref{eq:energy-decomposition} and $p\geq c$ give $\iint V(y,y')\dd\mu\dd\mu<\infty$. Fubini's theorem therefore supplies $y_0\in\R^s$ such that
\begin{align}
B_0:=\int V(y_0,y')\dd\mu_Y(y')<\infty.
\end{align}
Define
\begin{align}\label{eq:first-variation-potential}
W(x,y)=\int p(x,x')V(y,y')\dd\mu(x',y')
+\frac1{Z_\mu}\int g(y,y')\dd\mu_Y(y').
\end{align}
It is integrable with respect to $\mu$, and
\begin{align}\label{eq:comparison-potential-bound}
W(x,y_0)\leq C B_0+Z_\mu^{-1}\qquad(x\in K).
\end{align}

For a measurable $A\subset K\times\R^s$, write $\mu_A=\one_A\mu$ and let $\nu_A$ be its pushforward under $(x,y)\mapsto(x,y_0)$. The curve
\begin{align}
\mu_t=\mu+t(\nu_A-\mu_A),\qquad 0\leq t\leq1,
\end{align}
belongs to $\PX$. To differentiate $I_\rho(\mu_t)$ at zero, note that all terms obtained by expanding the integral of $pV$ are finite. The terms involving only $\mu$ and $\mu_A$ are bounded by $C\iint V\dd\mu\dd\mu$. A term involving one factor $\nu_A$ is bounded by $CB_0$, and the term with two factors $\nu_A$ vanishes because $V(y_0,y_0)=0$. The integrals involving $g$ are bounded, and $Z_\mu>0$. Expanding the quadratic terms and differentiating the logarithm therefore gives, by minimality,
\begin{align}\label{eq:redistribution-variation}
0\leq\left.\frac{d}{dt}I_\rho(\mu_t)\right|_{t=0+}
=2\int_A\bigl(W(x,y_0)-W(x,y)\bigr)\dd\mu(x,y).
\end{align}
Because this holds for every measurable $A$, $W(x,y)\leq W(x,y_0)$ for $\mu$-almost every $(x,y)$.

Choose $R_0<\infty$ with $m_0:=\mu_Y(\overline B(0,R_0))>0$. For $\norm y>R_0$, monotonicity of $v$ gives
\begin{align}\label{eq:potential-divergence}
W(x,y)\geq c m_0 v(\norm y-R_0)\longrightarrow\infty
\qquad(\norm y\to\infty),
\end{align}
uniformly in $x$. Comparing with \eqref{eq:comparison-potential-bound} shows that $\mu$ gives zero mass outside a sufficiently large output ball. Its support is consequently contained in the product of that closed ball and compact $K$.
\end{proof}

\begin{proposition}\label{prop:minimizer-stationarity}
Under Condition~\ref{output_kernel_conditions}, every compactly supported global minimizer $\mu$ of $I_\rho$ over $\PX$, with bounded $p$, belongs to $\SPX$.
\end{proposition}
\begin{proof}
For $\alpha\in C_c(K\times\R^s;\R^s)$, let $\mu_t$ be the law of $(x,y+t\alpha(x,y))$ under $\mu$, for $|t|\leq1$. Then $\mu_t\in\PX$. All output coordinates belong to a fixed compact set, on whose square $g$ is bounded away from zero and $\nabla_y g$ and $h$ are bounded. We may therefore differentiate under the integrals. Minimality and symmetry give
\begin{align}
0=\left.\frac{d}{dt}I_\rho(\mu_t)\right|_{t=0}
=-2\iint\alpha(x,y)\cdot\bigl(p(x,x')-q_\mu(y,y')\bigr)
 h(y,y')\dd\mu(x,y)\dd\mu(x',y').
\end{align}
The inner integral is a bounded vector field on the support of $\mu$. Since its integral against every such $\alpha$ vanishes, it is zero $\mu$-almost everywhere. Thus $\mu\in\SPX$.
\end{proof}

\subsection{Proofs of Theorems~\ref{tsne:main_theorem_1} and \ref{tsne:main_theorem_2}}\label{sec:main-proof}
\begin{proof}
Fix an input realization in the event of probability one on which $\nu_n\Rightarrow\mu_X$ and the conclusions of Proposition~\ref{prop:uniform-input}, Lemmas~\ref{lem:input-density} and \ref{lem:input-entropy}, and Proposition~\ref{prop:limsup} hold. Write $d_n=d_{n,\rho}(X)$ and $m=\inf_{\PX}I_\rho$. The first assertion of Theorem~\ref{tsne:main_theorem_1}, the uniqueness of the solutions of \eqref{eq:row-entropy}, is part of Proposition~\ref{prop:uniform-input}.

Let $\varepsilon_n\downarrow0$ and let $Y^{(n)}$ satisfy \eqref{eq:near-minima}. By \eqref{eq:collapsed-comparison} the values $L_{n,\rho}(X,Y^{(n)})$ are bounded, so Proposition~\ref{prop:output-tightness} provides translations $a_n$ making the measures $\widehat\mu_n$ in \eqref{eq:translated-empirical} tight. Choose a subsequence along which $d_n\to\liminf_nd_n$, and then a further subsequence along which $\widehat\mu_n\Rightarrow\mu$. Translating all the output points by the same vector does not change $L_{n,\rho}$, so Lemma~\ref{lem:liminf} together with \eqref{eq:near-minima} gives
\begin{align}
\liminf_n d_n\geq I_\rho(\mu)\geq m.
\end{align}
Combined with Proposition~\ref{prop:limsup}, this shows that $d_n\to m$ and $I_\rho(\mu)=m$; in particular, the infimum of $I_\rho$ over $\PX$ is attained. The same reasoning applies to any convergent subsequence of any translated sequence satisfying \eqref{eq:near-minima}: if $\widehat\mu_{n_k}\Rightarrow\lambda$, then $\lambda\in\PX$ and, by Lemma~\ref{lem:liminf}, $I_\rho(\lambda)\leq\lim_kd_{n_k}=m$, so $\lambda$ is a minimizer.

Every minimizer has finite energy, since $m\leq A_p$, so Proposition~\ref{prop:compact-support} shows that it has compact support. This proves Theorem~\ref{tsne:main_theorem_2}. Proposition~\ref{prop:minimizer-stationarity} then shows that every minimizer belongs to $\SPX$, and since $\SPX\subset\PX$,
\begin{align}
m\leq\inf_{\lambda\in\SPX}I_\rho(\lambda)\leq I_\rho(\mu)=m.
\end{align}
Hence the infimum over $\SPX$ is attained too, and the proof of Theorem~\ref{tsne:main_theorem_1} is complete.
\end{proof}

\section{Examples of Kernels}\label{tsne:sec:generality_conditions}
Condition~\ref{input_kernel_conditions} requires $t\mapsto tw'(t)$ to be nondecreasing. This gives the sign of the derivative of $F$ in Lemma~\ref{lem:entropy-monotone}; the integrability assumption is used in its large-$\sigma$ asymptotic. Both requirements hold for $w(t)=t$ and every $\theta>0$, giving Gaussian, exponential, and quartic input kernels when $\theta=2,1,4$, respectively. Polynomial kernels are obtained from
\begin{align}\label{eq:polynomial-example}
w(t)=a\log(1+t),\qquad
\ki(x,x',\sigma)=(1+\sigma\norm{x-x'}^\theta)^{-a}.
\end{align}
Here $w'(t)+tw''(t)=a/(1+t)^2>0$, and the integral in \eqref{eq:weighted-input-integrability} is finite exactly when $a\theta>d$. For these kernels, sufficient conditions for Condition~\ref{hyp:input} are given in Proposition~\ref{prop:global-calibration}.

For the Student-type output kernel,
\begin{align}\label{eq:student-tail}
k(r)=(1+r^2/\alpha)^{-\alpha},\qquad \alpha>0,
\end{align}
we have $k(\lambda R)/k(R)\to\lambda^{-2\alpha}$, so Condition~\ref{cond:tail}(a) holds when $2\alpha\geq s$. In particular, the Cauchy kernel ($\alpha=1$) satisfies this condition in two dimensions, whereas the choice $\alpha=1/2$ does not. Similarly, $(1+r^4)^{-1}$ satisfies it for $s\leq4$. The remaining kernels in Table~\ref{tab:theory-coverage} satisfy \eqref{eq:rapid-tail}. For $k(r)=e^{-r^2-0.8\sin(r^2)}$, for example,
\begin{align}
\log\frac{k(\lambda R)}{k(R)}
=-(\lambda^2-1)R^2+O(1)\longrightarrow-\infty
\qquad(\lambda>1).\notag
\end{align}
This kernel is decreasing because $1+0.8\cos(r^2)>0$. All kernels in the table also satisfy Condition~\ref{output_kernel_conditions}.

\begin{table}[htbp]
\centering\small
\begin{tabular}{@{}p{.39\linewidth}p{.22\linewidth}p{.31\linewidth}@{}}
\toprule
Output kernel & Type of decay & Restriction on $s$ \\
\midrule
$(1+r^2/\alpha)^{-\alpha}$ & Regular, $\beta=2\alpha$ & $s\leq2\alpha$ \\
$(1+r^4)^{-1}$ & Regular, $\beta=4$ & $s\leq4$ \\
$e^{-r^2}$ and $e^{-r^4}$ & Rapid & All $s$ \\
$e^{2-2\sqrt{1+r^2}}$ & Rapid & All $s$ \\
$e^{-r^2-0.8\sin(r^2)}$ & Rapid & All $s$ \\
\bottomrule
\end{tabular}
\caption{Examples satisfying Conditions~\ref{output_kernel_conditions} and \ref{cond:tail}.}\label{tab:theory-coverage}
\end{table}

\section{Numerical Experiments}\label{sec:numerical}

We compare two-dimensional visualizations of Gaussian noise, Gaussian mixtures,
MNIST, and Fashion-MNIST under different input and output kernels. Each point
in the figures is an output vector of the dimension reduction algorithm.
The comparisons keep the data and initial coordinates fixed while changing
the kernels, so that changes in the overall shape, gaps between groups, and
structure within a group can be seen directly.

\subsection{Data and numerical procedure}

The main examples contain $n=3000$ observations each. The synthetic data are
sampled from $N(0,I_{50})$ and an equally weighted mixture of three
$N(\mu_j,I_{50})$ distributions, where
\begin{align}
    \mu_j=\frac{6}{\sqrt{2}}
    \left(e_j-\frac13(e_1+e_2+e_3)\right),\qquad j=1,2,3. \notag
\end{align}
The pairwise distance between component means is $6$.

For MNIST \citep{lecun-mnist} and Fashion-MNIST \citep{xiao2017fashion}, we
sample $300$ images from each of the ten classes in the official test split.
Each $28\times28$ grayscale image is flattened to a vector in
$\mathbb{R}^{784}$, with pixel intensities divided by $255$. We use Euclidean
distances between these vectors, without a preliminary PCA or a learned
feature representation. Labels determine the balanced image subsets and
their colors; they do not enter the affinities or the optimization.

Writing $r=\norm{x-x'}$, the four input kernels are
\begin{align}
    \ki(x,x',\sigma)\in\left\{
    e^{-\sigma r^2},\ e^{-\sigma r},\ e^{-\sigma r^4},\
    (1+\sigma r^2)^{-(d/2+1)}\right\}.
    \label{eq:numerical-inputs}
\end{align}
We call them Gaussian, exponential, quartic, and polynomial, respectively.
The parameter $\sigma$ follows the convention in Condition
\ref{input_kernel_conditions}. For the main comparisons,
$\textbf{Perp}=150=0.05n$; each bandwidth is calibrated to an absolute
row-entropy error of at most $10^{-7}$. We also compare the image datasets
at $\textbf{Perp}=50$.

Table~\ref{tab:numerical-outputs} lists the output kernels. The Student-type family
varies the tail decay and contains the Cauchy kernel at $\alpha=1$; see also
\citep{10.1007/978-3-030-46150-8_8}. The two quartic kernels are flatter near
zero than the Gaussian and Cauchy kernels. The rippled Gaussian changes the
logarithmic slope while remaining decreasing for $r>0$, since
$1+0.8\cos(r^2)>0$.

\begin{table}[!ht]
    \centering
    \renewcommand{\arraystretch}{1.3}
    \begin{tabular}{ll}
        \hline
        Name & $k(r)$ \\
        \hline
        Student-type & $(1+r^2/\alpha)^{-\alpha}$,
        $\alpha\in\{1/2,1,2\}$ \\
        Gaussian & $e^{-r^2}$ \\
        Quartic exponential & $e^{-r^4}$ \\
        Quartic rational & $(1+r^4)^{-1}$ \\
        Smooth exponential & $e^{2-2\sqrt{1+r^2}}$ \\
        Rippled Gaussian & $e^{-r^2-0.8\sin(r^2)}$ \\
        \hline
    \end{tabular}
    \caption{The eight output kernels used in the simulations.}
    \label{tab:numerical-outputs}
\end{table}

Gaussian input with Cauchy output gives the t-SNE objective. Gaussian input
with Gaussian output gives symmetric SNE \citep{key}. We separately implement
original SNE \citep{hinton2002sne}, with conditional output probabilities:
\begin{align}
    q_{j|i}&=\frac{e^{-\norm{Y_i-Y_j}^2}}
    {\sum_{\ell\ne i}e^{-\norm{Y_i-Y_\ell}^2}}, \notag\\
    C_n&=\frac1n\sum_i\sum_{j\ne i}p_{j|i}
    \log\frac{p_{j|i}}{q_{j|i}}. \label{eq:numerical-sne}
\end{align}
The factor $1/n$ does not change its minimizers. Original and symmetric SNE
therefore differ even though both use Gaussian input and output kernels.

For each dataset, we cross all four inputs with Cauchy, Gaussian, Student-type
$\alpha=1/2$, quartic exponential, and quartic rational outputs. Original
SNE and the other three output kernels with Gaussian input bring the
comparison to $24$ configurations per dataset; $15$ are shown
($12$ in the main grid and $3$ in Figure~\ref{fig:other-outputs}). The displayed data use
sampling seed $101$ and initialization seed $501$, chosen before inspecting
the embeddings. For t-SNE, original SNE, symmetric SNE, Student-type $\alpha=1/2$,
and the two quartic outputs, we also use sampling seeds $202$ and $303$ and
initialization $502$ on the first sample. The image subsets may overlap.

The initial output coordinates are sampled from $N(0,0.1^2I_2)$ and centered.
Every method uses exact pairwise affinities in double precision, followed by
$1000$ Adam steps with a cosine learning-rate decrease from $0.15$ to $0.015$
and at most $400$ L-BFGS steps with a strong Wolfe line search. No early
exaggeration is used. Analytic gradients are checked against automatic
differentiation and directional finite differences.

These are finite-sample computations with a fixed optimization budget.
The heavier-tailed outputs can still change under further optimization.
The plots do not certify global minimizers or the asymptotic theorems, and
the Gaussian input laws have unbounded support. The Student-type kernel with
$\alpha=1/2$ is included as a finite-sample comparison; its decay exponent
is below the threshold in Condition~\ref{cond:tail} for two-dimensional
output. The finite image datasets also fall outside the continuous-density
hypothesis. Thus the experiments include cases beyond the scope of the
main theorems.

For display, coordinates are centered and orthogonally aligned to the
t-SNE embedding of the same sample. Distances are not rescaled. Each panel
has equal horizontal and vertical scales and its own coordinate range.
If remote points would obscure the main cloud, a central view is accompanied
by a labeled full-view inset and the number of points outside the central
view. Every observation is included in the full view.

\subsection{Shapes produced by different kernels}

Figure~\ref{fig:gaussian-grid} shows the output for a single Gaussian input
distribution. Cauchy output gives a roughly circular cloud with a denser
center. Gaussian output gives a more evenly filled disk. Quartic exponential
output produces three pronounced lobes, despite the absence of mixture
components in the input. This pattern appears for each input kernel in the
displayed sample. The corresponding columns look similar across inputs,
while changing the output kernel changes the cloud more visibly.

For the mixture in Figure~\ref{fig:mixture-grid}, all three displayed output
kernels separate the generating components. Cauchy output gives compact
groups with relatively large gaps. Gaussian output gives groups that are
broader relative to the distances between them. Quartic exponential output creates several concentrations within
each component, most clearly with Gaussian, exponential, and polynomial
input. A single population can therefore acquire visible internal
subdivisions in the output. Colors identify generating components, not
clusters fitted to the two-dimensional configuration.

\begin{samepage}
The MNIST outputs in Figure~\ref{fig:mnist-grid} contain recognizable regions
dominated by individual digits, but their boundaries depend on the output
kernel. With Cauchy output, the regions are more separated; $4$ and $9$
remain close or partly mixed, as do $3$, $5$, and $8$. Gaussian output
compresses the arrangement into a more continuous cloud. Quartic exponential
output can gather several digit classes into common lobes. For example,
panel (f) groups much of the $0$ and $6$ data near the top, while $4$, $7$,
and $9$ occupy a common region on the left. A distinct lobe need not
correspond to one digit.
\end{samepage}

Fashion-MNIST gives a different geometry (Figure~\ref{fig:fashion-grid}).
With Cauchy output, trousers form a distinct group, and footwear occupies
several neighboring groups. Pullovers, coats, and shirts overlap more
strongly. Gaussian output spreads these groups along a more continuous
configuration. Quartic exponential output gives a roughly rounded
arrangement with class-dependent concentrations around its boundary.
The classes that remain mixed thus differ between the two image datasets.

Figure~\ref{fig:other-outputs} compares original conditional SNE with two
further output kernels, all using Gaussian input. Comparison with the
Gaussian-output columns above separates the effect of conditional
normalization from the choice of kernel. Student-type $\alpha=1/2$ output can
produce compact groups together with remote points. Quartic rational output
can produce smaller concentrations within a larger cloud. The axis ranges
also matter: similar-looking groups can occupy very different output scales.

Figure~\ref{fig:perplexity-embeddings} holds the t-SNE kernels, sample, and
initialization fixed and changes perplexity from $150$ to $50$. The gaps and
local concentrations change, and a few remote points can substantially
enlarge the full coordinate range. Kernel comparisons should therefore be
made at a fixed perplexity, as in each preceding grid.

\begin{figure}[p]
    \centering
    \includegraphics[width=\textwidth]{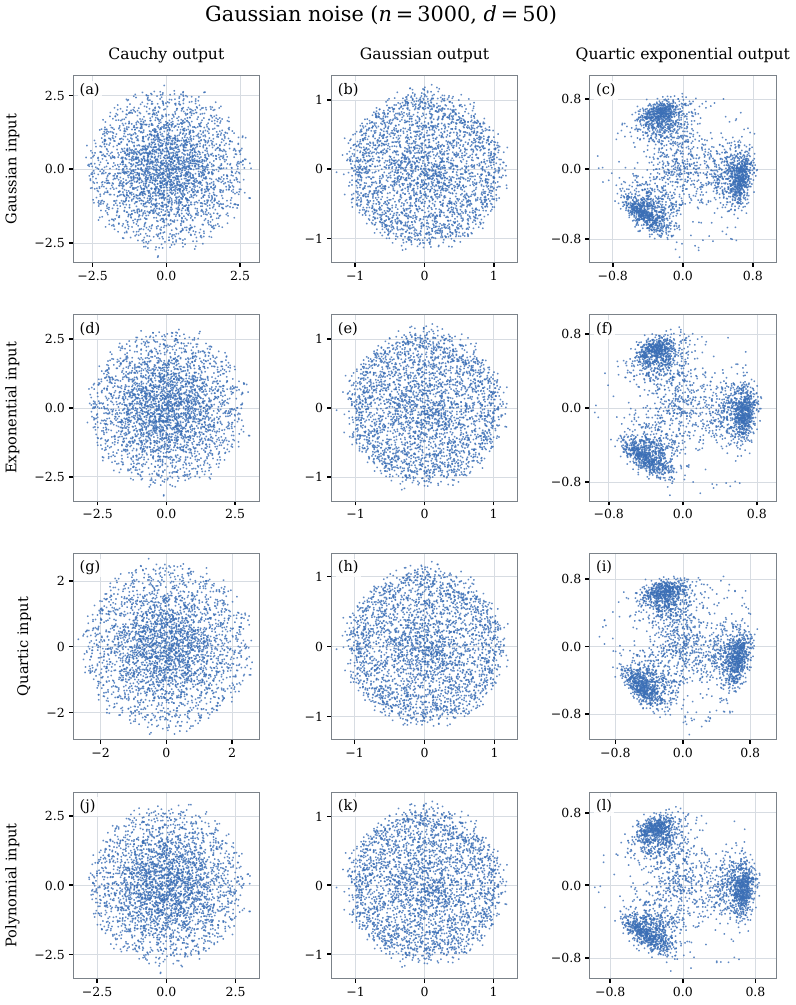}
    \caption{Two-dimensional outputs for $3000$ observations from
    $N(0,I_{50})$, with $\textbf{Perp}=150$. Rows vary the input kernel;
    columns vary the output kernel. The same observations and initial
    coordinates are used in every panel. Panel (a) is t-SNE and panel (b)
    is symmetric SNE. Axes show output coordinates after centering and
    orthogonal alignment, without a change of scale.}
    \label{fig:gaussian-grid}
\end{figure}

\begin{figure}[p]
    \centering
    \includegraphics[width=\textwidth]{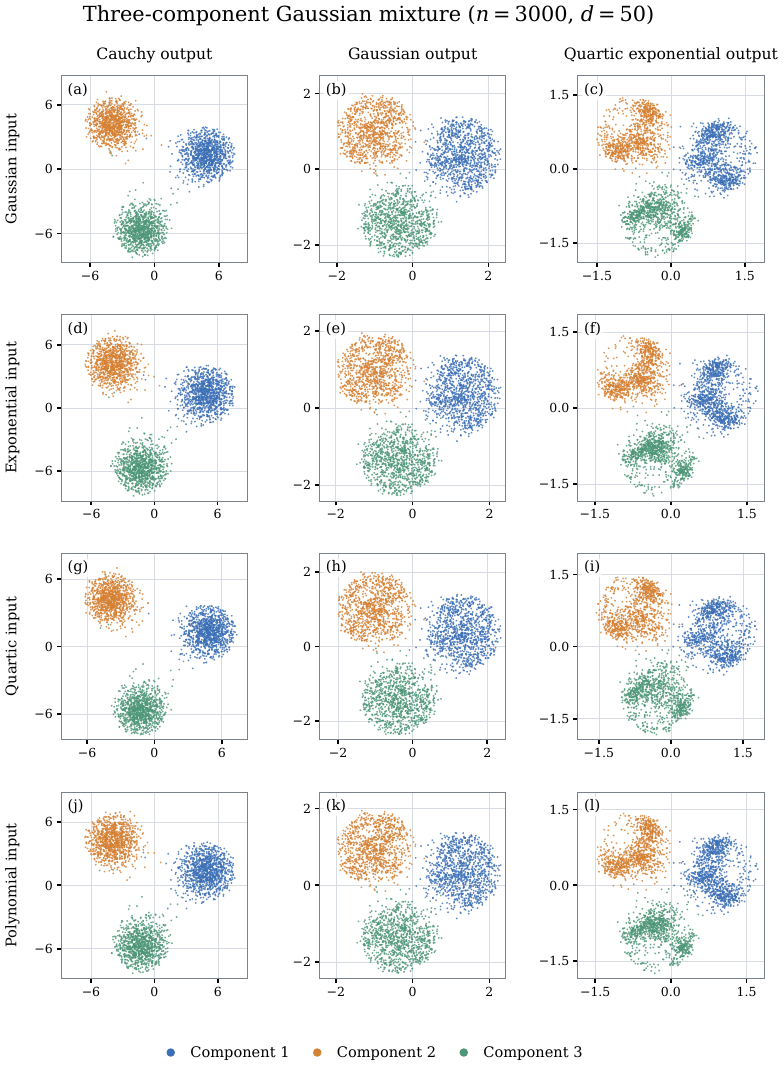}
    \caption{The same kernel combinations as in
    Figure~\ref{fig:gaussian-grid}, for the balanced three-component Gaussian
    mixture in $\mathbb{R}^{50}$. There are $3000$ observations and
    $\textbf{Perp}=150$. Colors identify the generating component.
    All panels reuse the same sample and initialization.}
    \label{fig:mixture-grid}
\end{figure}

\begin{figure}[p]
    \centering
    \includegraphics[width=\textwidth]{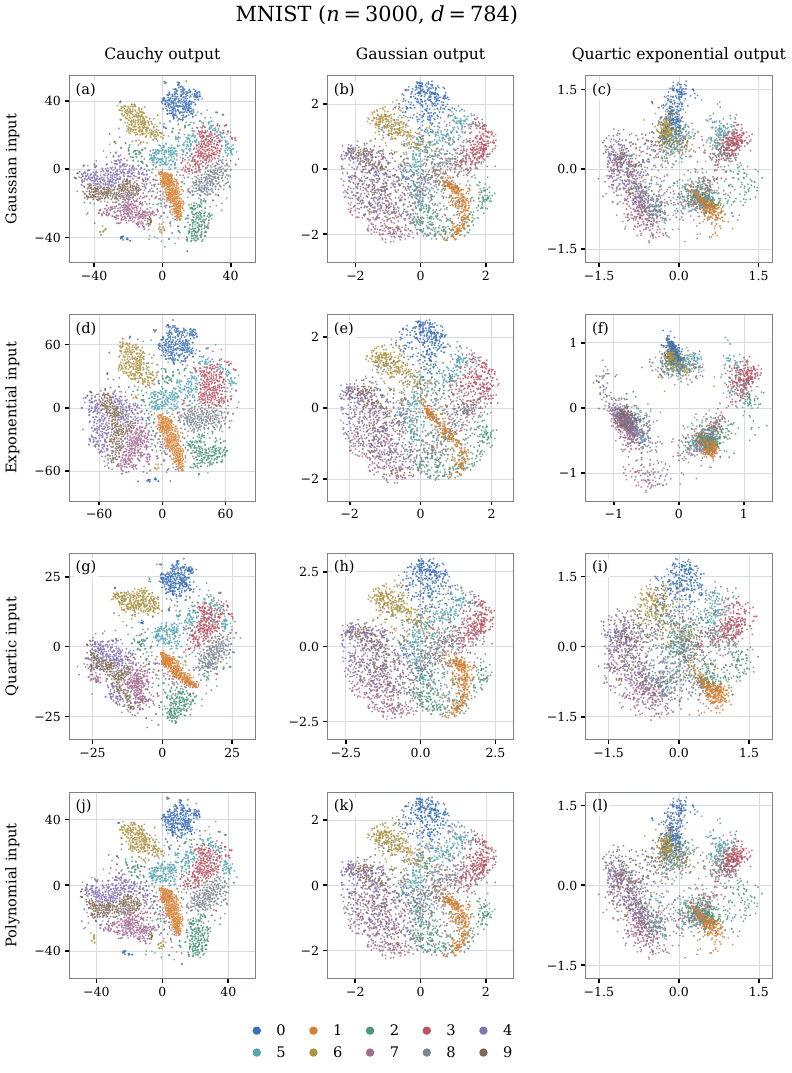}
    \caption{MNIST: the same $3000$ images, represented by their $784$ pixel
    values, embedded in two dimensions under twelve kernel combinations.
    There are $300$ images per digit and $\textbf{Perp}=150$.
    Colors denote digit labels and are fixed across panels. The layout
    follows Figure~\ref{fig:gaussian-grid}.}
    \label{fig:mnist-grid}
\end{figure}

\begin{figure}[p]
    \centering
    \includegraphics[width=\textwidth]{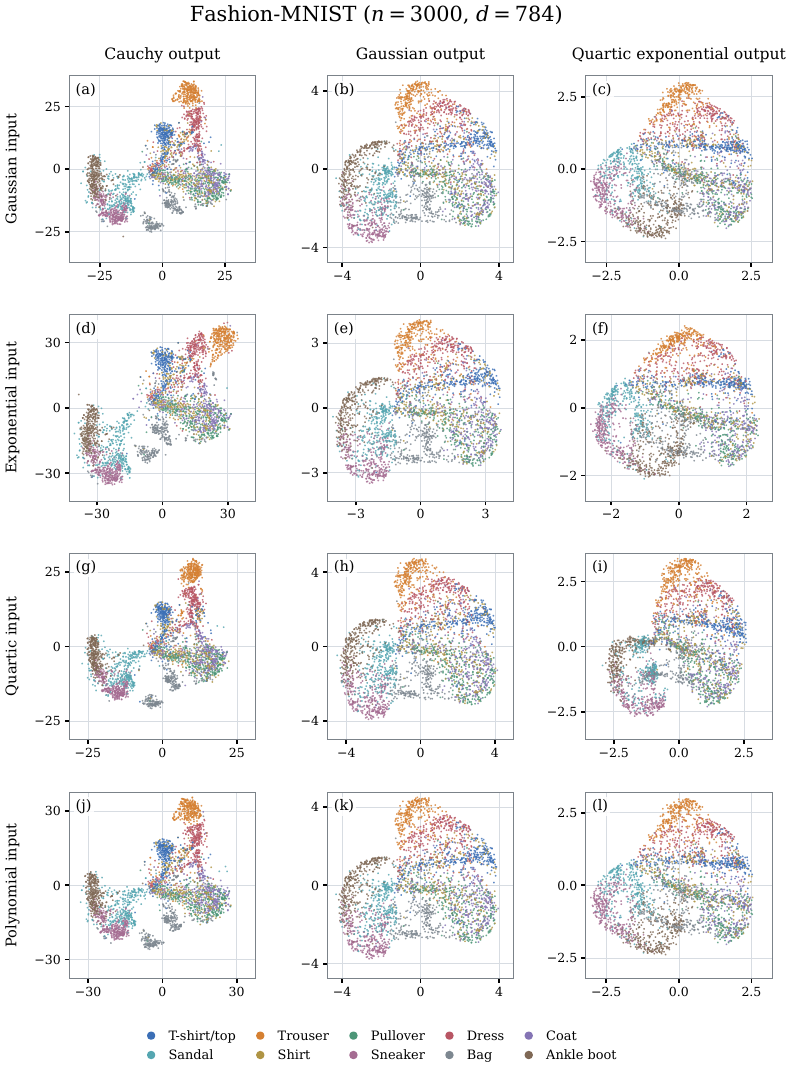}
    \caption{Fashion-MNIST: $3000$ images in $\mathbb{R}^{784}$, with
    $300$ images per clothing class and $\textbf{Perp}=150$.
    Rows, columns, initialization, and display conventions match
    Figure~\ref{fig:mnist-grid}. Class colors are shown in the legend.}
    \label{fig:fashion-grid}
\end{figure}

\begin{figure}[p]
    \centering
    \includegraphics[width=\textwidth]{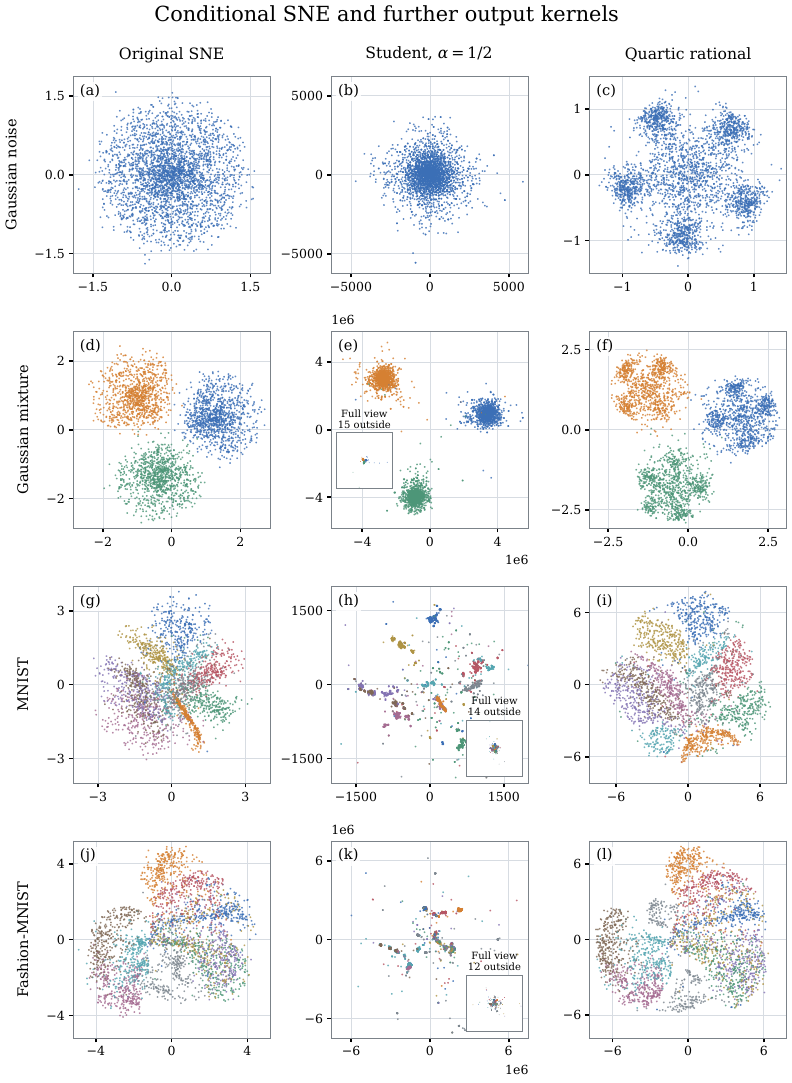}
    \caption{Original SNE, Student-type $\alpha=1/2$ output, and quartic rational
    output on the four main datasets. All panels use Gaussian input,
    $n=3000$, and $\textbf{Perp}=150$. The first column uses
    \eqref{eq:numerical-sne}; the others use the joint objective.
    Colors agree with the preceding figures. Full-view insets, where
    present, include remote points outside the central view.}
    \label{fig:other-outputs}
\end{figure}

\begin{figure}[p]
    \centering
    \includegraphics[width=\textwidth]{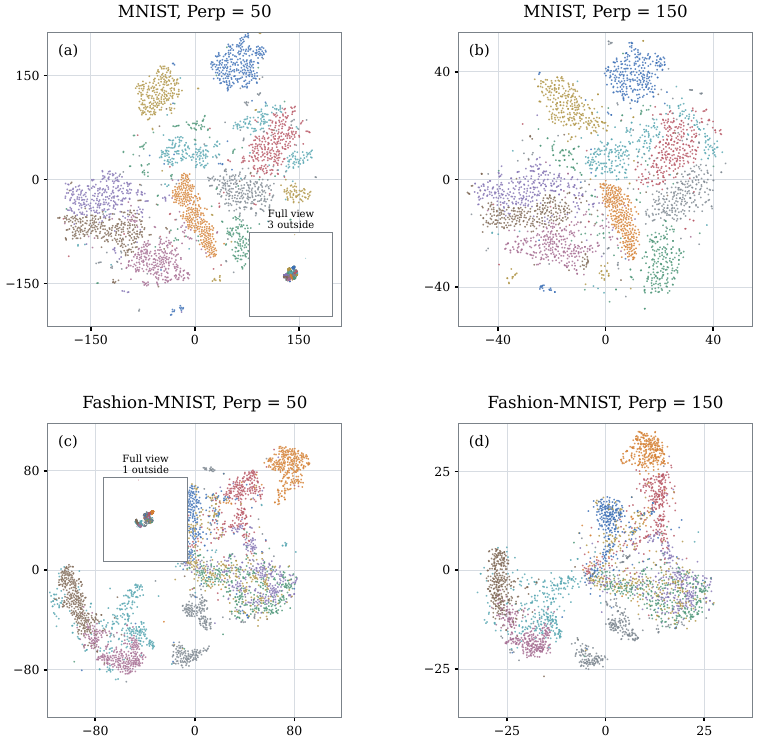}
    \caption{Changing perplexity with the t-SNE kernels fixed. Each row uses
    the same $3000$ image vectors and initial coordinates as the corresponding
    main grid. The central views and any full-view insets retain the original
    output scales. Class colors follow Figures~\ref{fig:mnist-grid} and
    \ref{fig:fashion-grid}.}
    \label{fig:perplexity-embeddings}
\end{figure}

\clearpage

\section{Acknowledgements}
This research was supported in part by grants from the NSF (DMS-2235451) and Simons Foundation (MPS-NITMB-00005320) to the NSF-Simons National Institute for Theory and Mathematics in Biology (NITMB).

The first author would like to thank Xiaoyu Li for his unwavering support through the years during the difficult times.
\section{Declarations}
The mathematical ideas, proof strategies, and original proofs underlying this work were developed by the authors without the use of AI tools. After completion of the original manuscript, AI tools were used to assist with grammar correction, polishing the writing, reviewing mathematical correctness, and subsequent revisions. The authors retained control over the mathematical arguments and the final manuscript. The authors assume responsibility for all content.

AI was also used for the writings and the Python programs that produce the visualization plots in section \ref{sec:numerical}. Please note that section \ref{sec:numerical} contains only visualizations and related setup information and does not contribute to the main mathematical proofs of this paper.

\bibliographystyle{hunsrtnat}
\bibliography{citation}

\appendix
\section{A Counterexample}\label{sec:calibration-counterexample}
We give a polynomial input kernel satisfying Condition~\ref{input_kernel_conditions} and a continuous compactly supported density for which Condition~\ref{hyp:input} fails. The equation $F(x,\sigma)=0$ has no solution at a zero of the density, and its solutions at nearby points are unbounded.
\begin{example}\label{ex:calibration-failure}
Take $d=3$, $s=2$, $\theta=2$, $\rho=0.05$, and
\begin{align}
w(t)=\frac52\log(1+t),\qquad
f(x)=\frac{63}{\pi}\norm x^4(1-\norm x)^2\one_{\{\norm x\leq1\}}.
\end{align}
The kernel satisfies Condition~\ref{input_kernel_conditions}, and $f$ is a continuous probability density with compact support. Nevertheless $F_{\rho,\mu_X}(0,\sigma)<0$ for every finite $\sigma>0$, so Condition~\ref{hyp:input} fails.
\end{example}
\begin{proof}
Here $w'>0$, $w'+tw''=(5/2)/(1+t)^2>0$, and the weighted radial integral in Condition~\ref{input_kernel_conditions} converges. The radial input density is $252r^6(1-r)^2$ on $(0,1)$, which integrates to one. The origin belongs to the support, though $f(0)=0$.

At $x=0$, multiply the kernel by $\sigma^{5/2}$ and write
\begin{align}
a_\sigma(r)=(r^2+\sigma^{-1})^{-5/2}\longrightarrow r^{-5}.
\end{align}
The limiting normalizer against the input law is
\begin{align}
A_\infty=252\int_0^1r(1-r)^2\dd r=21.
\end{align}
Thus the limiting tilted density relative to $\mu_X$ is $r^{-5}/21$, and its radial law is $12r(1-r)^2\dd r$. Since
\begin{align}
12\int_0^1r(1-r)^2\log r\dd r=-\frac{13}{12},
\end{align}
the relative entropy of the tilt has the finite limit
\begin{align}
D_\infty=\frac{65}{12}-\log21\simeq2.37214423.
\end{align}
Dominated convergence is valid for the normalizer and the entropy integrand: for $\sigma\geq1$, $a_\sigma\leq r^{-5}$, and $a_\sigma|\log a_\sigma|$ is bounded by $r^{-5}(C+5|\log r|)$, which is integrable against the radial input law. Consequently
\begin{align}
\lim_{\sigma\to\infty}F_{\rho,\mu_X}(0,\sigma)
=\frac{65}{12}-\log21+\log(0.05)
\simeq-0.62358804<0.
\end{align}
Strict monotonicity therefore precludes a root at zero. At every nonzero point in the interior of the unit ball the density is positive, so Proposition~\ref{prop:local-bandwidth} gives a finite root. Those roots diverge as $x\to0$: for each fixed $S<\infty$, continuity and $F(0,S)<0$ give $F(x,S)<0$ nearby, forcing the root to exceed $S$.
\end{proof}
Thus continuity of the density alone does not provide the uniform bandwidth bound used in Proposition~\ref{prop:uniform-input} for polynomial inputs.

\end{document}